\documentclass{article}

\usepackage{times}
\usepackage{graphicx} 
\usepackage{subfig} 

\usepackage{tikz}
\usepackage{pgfplots}
\usetikzlibrary[pgfplots.colormaps]
\newlength\figureheight
\newlength\figurewidth

\usepackage{natbib}

\usepackage{algorithm}
\usepackage{algorithmic}

\usepackage{hyperref}
\usepackage{cleveref}

\usepackage{proceed2e}
\usepackage[margin=1in]{geometry}

\usepackage{xspace}
\usepackage{amsmath, amssymb,amsfonts,amsthm}
\usepackage{bbm}

\newtheorem{theorem}{Theorem}
\newtheorem{lemma}{Lemma}
\newtheorem{cor}{Corollary}

\newcommand{\selfsparring}{{\sc\textsf{SelfSparring}}\xspace}
\newcommand{\multisparring}{{\sc\textsf{IndSelfSparring}}\xspace}

\newcommand{\kersparring}{{\sc\textsf{KernelSelfSparring}}\xspace}

\newcommand{\argmax}{\operatornamewithlimits{argmax}}

\newcommand{\K}{\mathcal{K}}

\title{Multi-dueling Bandits with Dependent Arms}

 \author{ {\bf Yanan Sui} \\
Caltech          \\
Pasadena, CA 91125 \\
ysui@caltech.edu \\
 \And
 {\bf Vincent Zhuang}  \\
Caltech          \\
Pasadena, CA 91125 \\
vzhuang@caltech.edu \\
 \And
 {\bf Joel W.~Burdick}   \\
Caltech         \\
Pasadena, CA 91125 \\
jwb@robotics.caltech.edu \\
 \And
 {\bf Yisong Yue}   \\
Caltech          \\
Pasadena, CA 91125 \\
yyue@caltech.edu \\
}          

\begin{document}

\maketitle

\begin{abstract} 
The dueling bandits problem is an online learning framework for learning from pairwise preference feedback, and is particularly well-suited for modeling settings that elicit subjective or implicit human feedback.
In this paper, we study the problem of \textit{multi-dueling bandits with dependent arms}, which extends the original dueling bandits setting by  simultaneously dueling multiple arms as well as modeling dependencies between arms.  These extensions capture key characteristics found in many real-world applications, and allow for the opportunity to develop significantly more efficient algorithms than were possible in the original setting.
We propose the \selfsparring algorithm, which reduces the multi-dueling bandits problem to a conventional bandit setting that can be solved using a stochastic bandit algorithm such as Thompson Sampling, and can naturally model dependencies using a Gaussian process prior. We  present a no-regret analysis for multi-dueling setting, and demonstrate the effectiveness of our algorithm empirically on a wide range of simulation settings.
\end{abstract}

\section{Introduction}

In many online learning settings, particularly those that involve human feedback, reliable feedback is often limited to pairwise preferences (e.g., ``is A better than B?''). Examples include implicit or subjective feedback for information retrieval and various recommender systems  \citep{chapelle2012large,sui2014clinical}.  This setup motivates the dueling bandits problem \citep{yue2012k}, which formalizes the problem of online regret minimization via preference feedback.

The original dueling bandits setting ignores many real world considerations. For instance, in personalized clinical recommendation settings \citep{sui2014clinical}, it is often more practical for subjects to provide preference feedback on several actions (or treatments) simultaneously rather than just two.  Furthermore, the action space can be very large, possibly infinite, but often has a low-dimensional dependency structure.

In this paper, we address both of these challenges in a unified framework, which we call \textit{multi-dueling bandits with dependent arms}.  We extend the original dueling bandits problem by simultaneously dueling multiple arms as well as modeling dependencies between arms using a kernel.  Explicitly formalizing these real-world characteristics provides an opportunity to develop principled algorithms that are much more efficient than algorithms designed for the original setting.  For instance, most dueling bandits algorithms suffer regret that scales linearly with the number of arms, which is not practical when the number of arms is very large or infinite.  

For this setting, we propose the \selfsparring algorithm, inspired by the Sparring algorithm from \cite{ailon2014reducing}, which algorithmically reduces the multi-dueling bandits problem into a conventional muilti-armed bandit problem that can be solved using a stochastic bandit algorithm such as Thompson Sampling \citep{chapelle2011empirical,russo2014learning}. Our approach can naturally incorporate dependencies using a Gaussian process prior with an appropriate kernel.

While there have been some prior work on multi-dueling \citep{brost2016multi} and learning from pairwise preferences over kernels \citep{gonzalez2016bayesian}, to the best of our knowledge, our approach is the first to address to both in a unified framework.  We are also the first to provide a regret analysis of the multi-dueling setting.  We further demonstrate the effectiveness of our approach over conventional dueling bandits approaches in a wide range of simulation experiments.

\section{Background}

\subsection{Dueling Bandits}
\label{sec:db}
The original dueling bandits problem is a sequential optimization problem with relative feedback. 
Let $\mathcal{B} = \{b_1,\ldots,b_K\}$ be the set of $K$ bandits (or arms). At each iteration, the algorithm duels or compares a single pair of arms $b_i, b_j$ from the set of $K$ arms ($b_i$ and $b_j$ can be identical). The outcome of each duel between $b_i$ and $b_j$ is an independent sample of a Bernoulli random variable. We define the probability that arm $b_i$ beats $b_j$ as: $$P(b_i \succ b_j) = \phi(b_i,b_j) + 1/2,$$
where $\phi(b_i,b_j)\in [-1/2,1/2]$ denotes the stochastic preference between $b_i$ and $b_j$, thus $b_i \succ b_j \Leftrightarrow \phi(b_i,b_j) > 0$.
We assume there is a total ordering, and WLOG that $b_i \succ b_j \Leftrightarrow i < j$.

The setting proceeds in a sequence of iterations or rounds. At each iteration $t$, the decision maker must choose a pair of bandits $b_t^{(1)}$ and $b_t^{(2)}$ to compare, and observes the outcome of that comparison. The quality of the decision making is then quantified using a notion of cumulative regret of $T$ iterations:
\begin{eqnarray}
R_T = \sum_{t=1}^T \left[ \phi(b_1,b_t^{(1)}) + \phi(b_1, b_t^{(2)})\right].\label{eqn:regret}
\end{eqnarray}
When the algorithm has converged to the best arm $b_1$, then it can simply duel $b_1$ against itself, thus incurring no additional regret.
In the recommender systems setting, one can interpret \eqref{eqn:regret} as the how much the user(s) would have preferred the best bandit over the the ones presented by the algorithm.  



To date, there have been several algorithms proposed for the stochastic dueling bandits problem, including Interleaved Filter \citep{yue2012k}, Beat the Mean \citep{yue2011beat}, SAVAGE \citep{urvoy2013generic}, RUCB \citep{zoghi2014relative,zoghi2015mergerucb}, Sparring \citep{ailon2014reducing,dudik2015contextual}, RMED \citep{komiyama2015regret}, and DTS \citep{wu2016doublets}.  Our proposed approach, \selfsparring, is inspired by Sparring, which along with RUCB-style algorithms are the best performing methods.  In contrast to Sparring, which has no theoretical guarantees, we provide no-regret guarantees for \selfsparring, and demonstrate significantly better performance in the multi-dueling setting. 

Previous work on extending the original dueling bandits setting have been largely restricted to settings that duel a single pair of arms at a time. These include continuous-armed convex dueling bandits \citep{yue2009interactively}, contextual dueling bandits which also introduces the von Neumann winner solution concept \citep{dudik2015contextual}, sparse dueling bandits that focuses on the Borda winner solution concept \citep{jamieson2015sparse}, Copeland dueling bandits that focuses on the Copeland winner solution concept \citep{zoghi2015copeland}, and adversarial dueling bandits \citep{gajane2015relative}.
In contrast, our work studies the complementary directions of how to formalize multiple duels simultaneously, as well as how to reduce the dimensionality of modeling the action space using a low-dimensional similarity kernel. 

Recently, there have been increasing interest in studying personalization settings that simultaneously elicit multiple pairwise comparisons.  Example settings include information retrieval \citep{hofmann2011probabilistic,schuth2014multileaved,schuth2016multileave} and clinical treatment \citep{sui2014clinical}. There have also been some previous work on multi-dueling bandits settings  \citep{brost2016multi,sui2014clinical,schuth2016multileave}, however the previous approaches are limited in their scope and lack rigorous theoretical guarantees.  In contrast, our approach can handle a wide range of multi-dueling mechanisms, has near-optimal regret guarantees, and can be easily composed with kernels to model dependent arms.


\subsection{Multi-armed Bandits}
Our proposed algorithm,
\selfsparring, utilizes a  multi-armed bandit (MAB) algorithm as a subroutine, and so we provide here a brief formal description of the conventional MAB problem for completeness.
The stochastic MAB problem \citep{robbins52} refers to an iterative decision making problem where the algorithm repeatedly chooses among K actions (or bandits or arms).  In contrast to the dueling bandits setting, where the feedback is relative between two arms, here, we receive an absolute reward that depends on the arm selected. We assume WLOG that every reward is bounded between $[0,1]$.\footnote{So long as the rewards are bounded, one can shift and re-scale them to fit within $[0,1]$.} The goal then is to minimize the cumulative regret compared to the best arm:
\begin{eqnarray}
R_T^{\text{MAB}} = \sum_{t=1}^T \left[\mu^1 - \mu(b_t)\right],
\label{eqn:mab_regret}
\end{eqnarray}
where $b_t$ denotes the arm chosen at time $t$, $\mu(b)$ denotes the expected reward of arm $b$, and $\mu^1 = \argmax_b \mu(b)$.
Popular algorithms for the stochastic setting include UCB (upper confidence bound) algorithms \citep{auer2002finite}, and Thompson Sampling
\citep{chapelle2011empirical,russo2014learning}.  

In the adversarial setting, the rewards are chosen in an adversarial fashion, rather than sampled independently from some underlying distribution.  In this case, regret \eqref{eqn:mab_regret} is rephrased as the difference in the sum of rewards. The predominant algorithm for the adversarial setting is EXP3 \citep{auer2002nonstochastic}.

\subsection{Thompson Sampling}
\label{sec:ts}

The specific MAB algorithm used by our \selfsparring approach is Thompson Sampling.
Thompson Sampling is a stochastic algorithm that maintains a distribution over the arms, and chooses arms by sampling  \citep{chapelle2011empirical}. This distribution is updated using reward feedback.  The entropy of the distribution thus corresponds to  uncertainty regarding which is the best arm, and flatter distributions lead to more exploration.

\newcommand{\muh}{\hat{\mu}}

\begin{algorithm}[t]
    \caption{Thompson Sampling for Bernoulli Bandits}
    \label{alg:ts}
{
\begin{algorithmic}[1]
	\STATE For each arm $i=1,2,\cdots, K$, set $S_i=0$, $F_i=0$.
    \FOR{$t=1,2,\ldots $}
    	\STATE For each arm $i=1,2,\cdots, K$, sample $\theta_{i}$ from $Beta(S_i+1,F_i+1)$
        \STATE Play arm $i(t) := \argmax_i{\theta_i(t)}$, observe reward $r_t$
        \STATE $S_i \leftarrow S_i + r_t$, $F_i \leftarrow F_i + 1 - r_t$
    \ENDFOR
\end{algorithmic}
}
\end{algorithm}

Consider the Bernoulli bandits setting where observed rewards are either 1 (win) or 0 (loss).  Let $S_i$ and $F_i$ denote the historical number of wins and losses of arm $i$, and let $D_t$ denote the set of all parameters at round $t$:
$$D_t= \{ S_1, \cdots, S_K; F_1, \cdots, F_K\}_t.$$
For brevity, we often represent $D_t$ by $D$, since only the current iteration matters at run-time. The sampling process of Beta-Bernoulli Thompson Sampling given $D$ is:
\begin{itemize}
\vspace{-0.1in}
\item For each arm $i$, sample $\theta_i \sim Beta(S_i+1,F_i+1)$.
\vspace{-0.1in}
\item Choose the arm with maximal $\theta_i$.
\end{itemize}
In other words, we model the average utility of each arm using a Beta prior, and rewards for arm $i$ as Bernoulli distributed according to latent mean utility $\theta_i$.  As we observe more rewards, we can compute the posterior, which is also Beta distributed by conjugation between Beta and Bernoulli.  The sampling process above can be shown to be sampling for the following distribution:
\begin{eqnarray}
P(i|D) = P(i =\argmax_b \theta_b|D).
\label{eqn:ts}
\end{eqnarray}
Thus, any arm $i$ is chosen with probability that it has  maximal reward under the Beta posterior. Algorithm~\ref{alg:ts} describes the Beta-Bernoulli Thompson Sampling algorithm, which we use as a subroutine for our approach.
Thompson Sampling enjoys near-optimal regret guarantees in the stochastic MAB setting, as given by the lemma below (which is a direct consequence of main theorems in \citet{agrawal2012,kaufmann2012}).

\begin{lemma}
	\label{lem:ts}
For the K-armed stochastic MAB problem, Thompson Sampling has expected regret:
$\mathbb{E}[R_T^{\text{MAB}}] = \mathcal{O}\left(\frac{K}{\Delta}\ln T \right)$, where $\Delta$ is the difference between expected rewards of the best two arms.
\end{lemma}


\subsection{Gaussian Processes \& Kernels}
\label{sec:gp}

 Normally, when one observes measurements about one arm (in both dueling bandits and conventional multi-armed bandits), one cannot use that measurement to infer anything about other arms -- i.e., the arms are independent.  This limitation necessarily implies that regret scales linearly w.r.t. the number of arms $K$, since each arm must be explored at least once to collect at least one measurement about it.  We will use Gaussian processes and kernels to model dependencies between arms.

For simplicity, we present Gaussian processes in the context of multi-armed bandits.  We will describe how to apply them to multi-dueling bandits in Section \ref{sec:problem}
A Gaussian process (GP) is a probability measure over functions such that any linear restriction is multivariate Gaussian. A GP is fully determined by its mean and a positive definite covariance operator, also known as a kernel. 
A $GP(\mu(b),k(b, b'))$ is a probability distribution across a class of ``smooth'' functions, which is parameterized by a kernel function $k(b, b')$ that characterizes the smoothness of $f$.  One can think of $f$ has corresponding to the reward function in the standard MAB setting. 

We assume WLOG~that $\mu(b)=0$, and that our observations are perturbed by i.i.d.~Gaussian noise, i.e., for samples at points $A_T=[b_1 \dots b_T]$, we have $y_t = f(b_t) + n_t$ where $n_t \sim T(0, \sigma^2)$ (e will relax this later). The posterior over $f$ is then also Gaussian with mean $\mu_T( b)$, covariance $k_T(b, b)$ and variance $\sigma_T^2(b, b')$ that satisfy:
\begin{align*}
\mu_T(b) &= k_T(b)^T(\K_T + \sigma^2I)^{-1}y_T\\
k_T(b, b') &= k(b, b') - k_T(x)^T(\K_T + \sigma^2I)^{-1}k_T(b')\\
\sigma_T^2(b) &= k_T(b, b),
\end{align*}
where $k_T(b) = [k(b_1, b) \dots k(b_T, b)]^T$ and $\K_T$ is the positive definite kernel matrix $[k(x, x')]_bx, b' \in A_T]$.

Posterior inference updates the mean reward estimates for all the arms that share dependencies (as specified by the kernel) with the arms selected for measurement.  Thus one can show that MAB algorithms using Gaussian processes have regret that scale linearly w.r.t. the dimensionality of the kernel rather than the number of arms (which can now be infinite) \citep{srinivas10}.


\section{Multi-dueling Bandits}
\label{sec:problem}

We now formalize the multi-dueling bandits problem. We inherit all notation from original dueling bandits setting (Section \ref{sec:db}).
The key difference is that the algorithm now selects a (multi-)set $S_t$ of arms at each iteration $t$, and observes outcomes of duels between some pairs of arms in $S_t$. For example, in information retrieval this can be implemented via multi-leaving \citep{schuth2014multileaved} the ranked lists of the subset, $S_t$, of rankers and then inferring the relative quality of the lists (and the corresponding rankers) from user feedback. 

In general, we assume the number of arms being dueled at each iteration is some fixed constant $m = |S_t|$.  When $m=2$, the problem reduces to the original dueling bandits setting. Extending the regret formulation from the original setting \eqref{eqn:regret}, we can write the regret as:
\begin{eqnarray}
R_T = \sum_{t=1}^T \sum_{b\in S_t} \phi(b_1,b).\label{eqn:regret2}
\end{eqnarray}


The goal then is to select subsets of arms ${S_t}$ so that the cumulative regret \eqref{eqn:regret2} is minimized. Intuitively, all arms have to be selected a small number of times in order to be explored, but the goal of the algorithm is to minimize the number of times when suboptimal arms are selected.
When the algorithm has converged to the best arm $b_1$, then it can simply choose $S_t$ to only contain $b_1$, thus incurring no additional regret.

Our setting differs from \citet{brost2016multi} in two ways.  First, we play a fixed, rather than variable, number of arms at each iteration. Furthermore, we focus on total regret, rather than the instantaneous average regret in a single iteration; in many applications (e.g., \citet{sui2014clinical}), playing each arm incurs its own regret .

\textbf{Feedback Mechanisms.} Simultaneously dueling multiple arms opens up multiple options for collecting feedback. For example, in some applications it may be viable to collect all pairwise feedback for all chosen arms  $S_t$.  In other applications, it is more realistic to only observe the ``winner'' of $S_t$, in which we observe feedback that one $b\in S_t$ wins against all other arms in $S_t$, but nothing about pairwise preferences between the other arms.  


\textbf{Approximate Linearity.}
One assumption that we leverage in developing our approach is \textit{approximate linearity}, which fully generalizes the linear utility-based dueling bandits setting studied in \citet{ailon2014reducing}.
For any triplet of bandits $b_i \succ b_j \succ b_k$ and some constant $\gamma > 0$: 
\begin{eqnarray}
\phi(b_i,b_k) - \phi(b_j,b_k) \geq \gamma\phi(b_i,b_j).
\label{eqn:al}
\end{eqnarray}

To understand Approximate Linearity, consider the special case when the preference function follows the form $\phi(b_i, b_j) = \Phi(u_i - u_j)$, where $u_i$ is a bounded utility measure of $b_i$. Approximate linearity of $\phi(\cdot, \cdot)$ is equivalent to having $\Phi(\cdot)$ be not far from some linear function on its bounded support (see Figure \ref{fig:al}), and is satisfied by any continuous monotonic increasing function.  When $\Phi$ is linear, then our setting reduces to the utility-based dueling bandits setting of \citet{ailon2014reducing}.\footnote{Compared to the assumptions of \cite{yue2012k}, Approximate Linearity is a stricter requirement than strong stochastic transitivity, and is a complementary requirement to stochastic triangle inequality. In particular, stochastic triangle inequality requires that the curve in Figure \ref{fig:al} exhibits diminishing returns in the top-right quadrant (i.e., is sub-linear), whereas Approximate Linearity requires that the curve be not too far from linear.}

\begin{figure}[t]
\centering
\includegraphics[width=0.92\columnwidth]{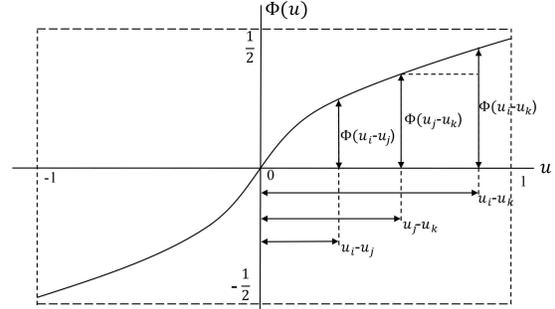}
\vspace{-0.2in}
\caption{Illustration of Approximate Linearity. The curve represents $\Phi(\cdot)$ with support on $[-1, 1]$.  Monotonicity guarantees Approximate Linearity for some $\gamma$.}
\label{fig:al}
\vskip -0.2in
\end{figure}

\section{Algorithms \& Results}
\label{sec:algorithm}

We start with a high-level description of our general framework, called \selfsparring, which is inspired by the Sparring algorithm from \citet{ailon2014reducing}.  The high-level strategy is to reduce the multi-dueling bandits problem to a multi-armed bandit (MAB) problem that can be solved using a MAB algorithm, and ideally lift existing MAB guarantees to the multi-dueling setting.

Algorithm \ref{alg:ss} describes the \selfsparring approach. 
\selfsparring uses a stochastic MAB algorithm such as Thompson sampling as a subroutine to independently sample the set of $m$ arms, $S_t$ to duel.  The distribution of $S_t$ is generally not degenerate (e.g., all the same arm) unless the algorithm has converged.   In contrast, the Sparring algorithm uses $m$ MAB algorithms to control the choice of the each arm, which essentially reduces the conventional dueling bandits problem to two multi-armed bandit problems ``sparring'' against each other. 


 \begin{algorithm}[tb]
     \caption{\selfsparring}
     \label{alg:ss}
 \begin{algorithmic}[1]
     \INPUT arms $1, \ldots, K$ in space $S$, $m$ the number of arms drawn at each iteration, $\eta$ the learning rate
     \STATE Set prior $D_0$ over $S$
     \FOR{$t=1,2,\ldots $}
       \FOR{$j = 1, \ldots, m$}
         \STATE select arm $i_j(t)$ using $D_{t-1}$ \label{lin:sample}
 	   \ENDFOR
	   \STATE Play $m$ arms $\{i_j(t)\}_j$ and observe $m\times m$ pairwise feedback matrix $R = \{r_{ij} \in \{0,1,\emptyset\}\}_{m \times m}$
       \STATE update $D_{t-1}$ using $R$ to obtain $D_t$
     \ENDFOR
 \end{algorithmic}
 \end{algorithm}

\selfsparring takes as input $S$ the total set of arms, $m$ the number of arms to be dueled at each iteration, and $\eta$ the learning rate for posterior updates. $S$ can be a finite set of $K$ arms for independent setting, or a  continuous action space of arms for kernelized setting. A prior distribution $D_0$ is used to initialize the sampling process over $S$. In the $t$-th iteration, \selfsparring selects $m$ arms by  sampling over the distribution $D_{t-1}$ as shown in line~\ref{lin:sample} of Algorithm~\ref{alg:ss}. The preference feedback can be any type of comparisons ranging from full comparison over the $m$ arms (a full matrix for $R$, aka `all pairs'') to single comparison of one pair (just two valid entries in $R$). The posterior distribution over arms $D_t$ then gets updated by $R$ and the prior $D_{t-1}$.

We specialize \selfsparring in two ways.  The first, \multisparring (Algorithm~\ref{alg:ms}), is the independent-armed version of \selfsparring.  The second,  \kersparring (Algorithm~\ref{alg:ks}), uses Gaussian processes to make predictions about preference function $f$ based on noisy evaluations over comparisons. We emphasize here that \selfsparring is very modular approach, and is thus easy to implement and extend.

\subsection{Independent Arms Case}
\label{sec:multisparring}
\multisparring (Algorithm \ref{alg:ms}) instantiates \selfsparring using Beta-Bernoulli Thompson sampling.  
The posterior Beta distributions  $D_t$ over the arms are updated by the preference feedback within the iteration and the prior Beta distributions $D_{t-1}$.

We present a no-regret guarantee of \multisparring in Theorem \ref{thm:ms} below.  We now provide a high-level outline of the main components leading to the result.  Detail proofs are deferred to the supplementary material.


Our first step is to prove that \multisparring is asymptotically consistent, i.e., it is guaranteed (with high probability) to converge to the best bandit.  In order to guarantee consistency, we first show that all arms are sampled infinitely often in the limit.

\begin{lemma}
\label{lem:io}
Running \multisparring with infinite time horizon will sample each arm infinitely often.
\end{lemma}

In other words, Thompson sampling style algorithms do not eliminate any arms. 
Lemma~\ref{lem:io} also guarantees concentration of any statistical estimates for each arm as $t\rightarrow \infty$.
We next show that the sampling of \multisparring will concentrate around the optimal arm.  

\begin{theorem}
\label{thm:conv}
Under Approximate Linearity, \multisparring converges to the optimal arm $b_1$ as running time $t\rightarrow \infty$: $\lim_{t\rightarrow \infty} \mathbb{P}(b_t = b_1) = 1$.
\end{theorem}

\begin{algorithm}[tb]
     \caption{\multisparring}
     \label{alg:ms}
 \begin{algorithmic}[1]
 	 \INPUT $m$ the number of arms drawn at each iteration, $\eta$ the learning rate
     \STATE For each arm $i=1,2,\cdots, K$, set $S_i=0$, $F_i=0$.
     \FOR{$t=1,2,\ldots $}
     	\FOR{$j = 1, \ldots, m$}
          \STATE For each arm $i=1,2,\cdots, K$, sample $\theta_{i}$ from $Beta(S_i+1,F_i+1)$
          \STATE Select $i_j(t) := \argmax_i{\theta_i(t)}$
    	\ENDFOR
        \STATE Play $m$ arms $\{i_j(t)\}_j$, observe pairwise feedback matrix $R = \{r_{jk} \in \{0,1,\emptyset\}\}_{m \times m}$
        \FOR{$j,k = 1, \ldots, m$}
          \IF{$r_{jk} \neq \emptyset$}
          \STATE $S_j \leftarrow S_j + \eta\cdot r_{jk}$, 
          $F_j \leftarrow F_j + \eta(1 - r_{jk})$
          \ENDIF
        \ENDFOR
     \ENDFOR
 \end{algorithmic}
 \end{algorithm}

\begin{figure*}[t!]
\centering
\subfloat[$5$ iterations]{\includegraphics[width=0.33\textwidth]{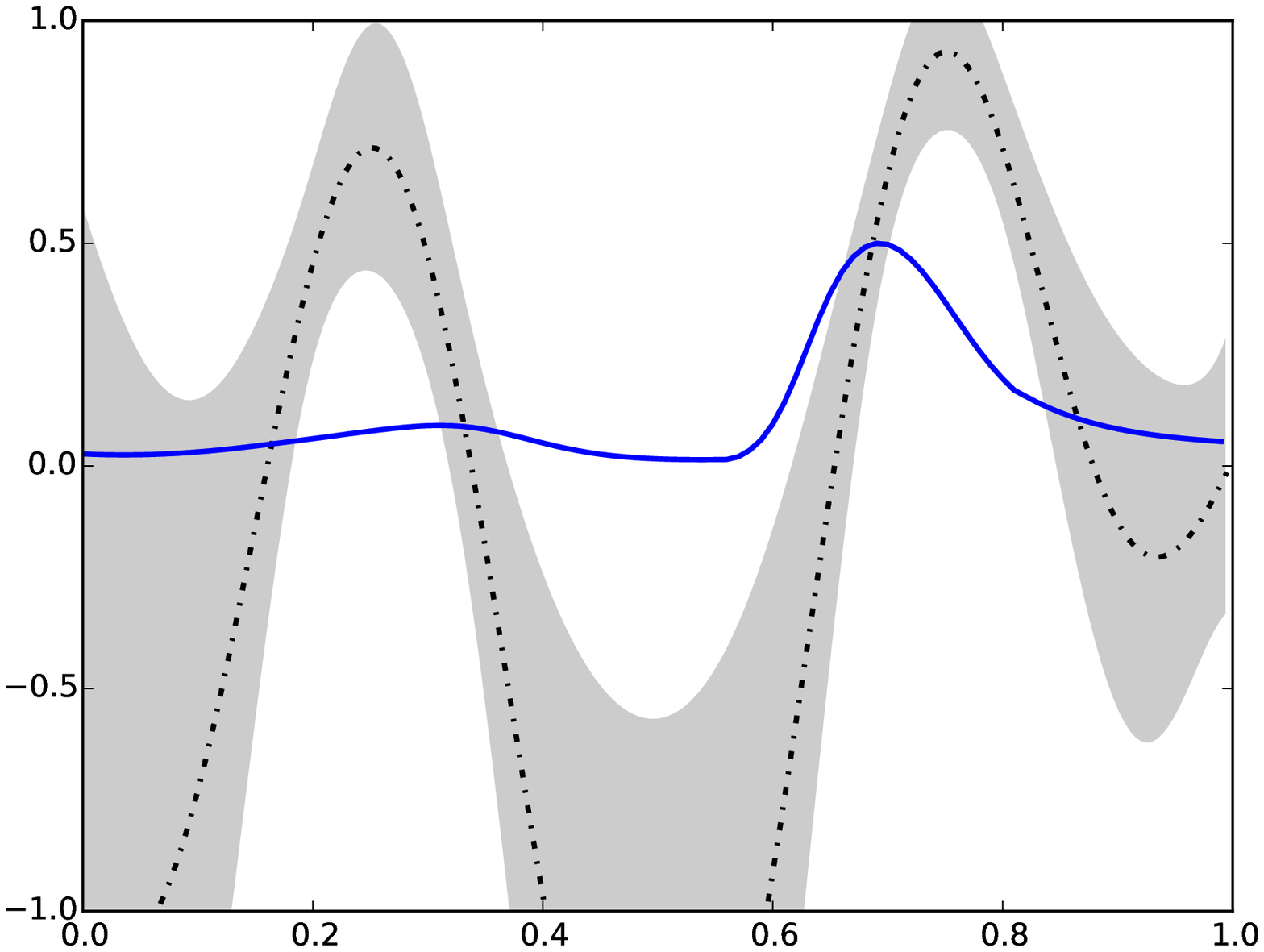}}
\subfloat[$20$ iterations]{\includegraphics[width=0.33\textwidth]{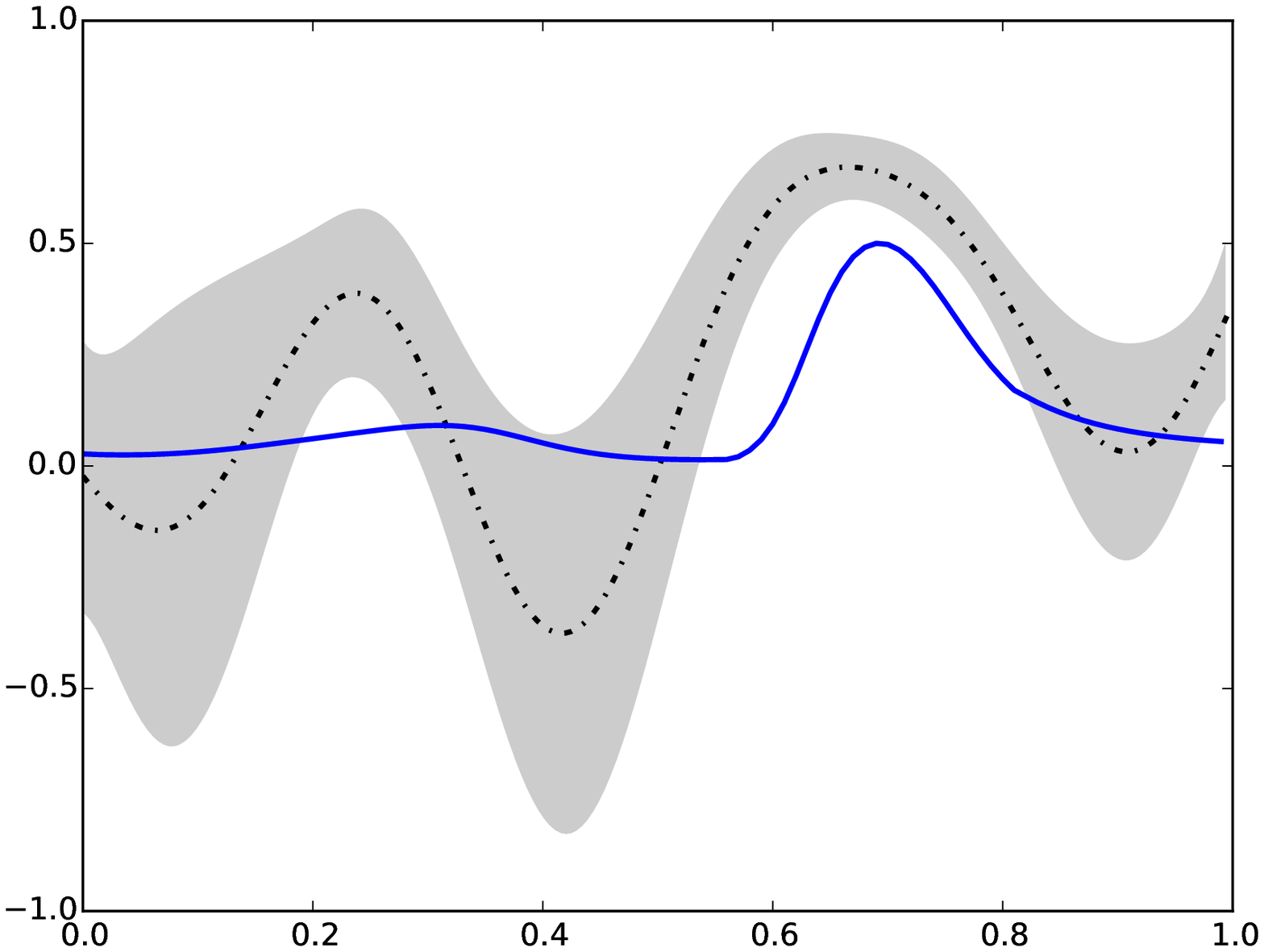}}
\subfloat[$100$ iterations]{\includegraphics[width=0.33\textwidth]{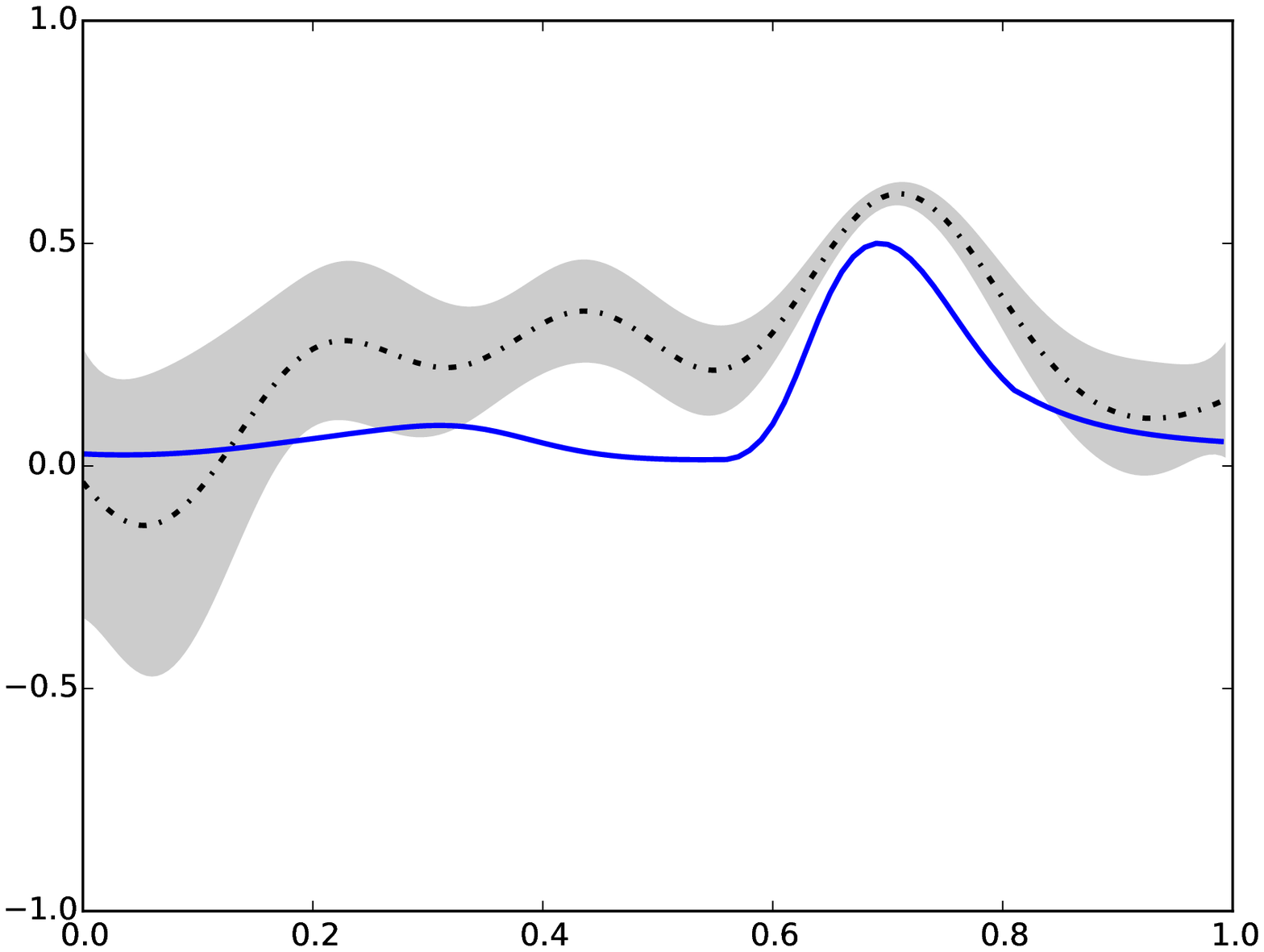}}
\caption{Evolution of a GP preference function in \kersparring; dashed lines correspond to the mean and shaded areas to $\pm 2$ standard deviations. The underlying utility function was sampled randomly from a GP with a squared exponential kernel with lengthscale parameter 0.2, and the resulting preference function is shown in blue. The GP finds the best arm with high confidence.}
\label{fig:ks}
\end{figure*}

Theorem~\ref{thm:conv} implies that \multisparring is asymptotically no-regret. As $t\rightarrow \infty$, the Beta distribution for each arm $i$ is converging to  $P(b_i \succ b_1)$, which implies converging to only choosing the optimal arm.

Most existing dueling bandits algorithm chooses one arm as a ``reference'' arm and the other arm as a competing arm for exploration/exploitation (in the $m=2$ setting). If the distribution over reference arms never changes, then the competing arm is playing against a fixed ``environment'', i.e., it is a standard MAB problem. For general $m$, we can analogously consider choosing only one arm against a fixed distribution over all the other arms.  Using Thompson sampling, the following lemma holds.

\begin{lemma}
	\label{lem:fixed}
Under Approximate Linearity, selecting only one arm  via Thompson sampling against a fixed distribution over the remaining arms leads to optimal regret w.r.t. choosing that arm.
\end{lemma}

Lemma~\ref{lem:fixed} and Theorem \ref{thm:conv} motivate the idea of analyzing the regret of each individual arm against near-fixed (i.e.,  converging) environments. 

\begin{theorem}
\label{thm:ms}
Under Approximate Linearity, \multisparring converges to the optimal arm with asymptotically optimal no-regret rate of $\mathcal{O}(K\ln(T)/\Delta)$.
\end{theorem}

Theorem \ref{thm:ms} shows an no-regret guarantee for \multisparring that asymptotically matches the optimal rate of $\mathcal{O}(K\ln(T)/\Delta)$ up to constant factors. In other words, once $t>C$ for some problem-dependent constant $C$, the regret of \multisparring matches information-theoretic bounds up to constant factors (see \citet{yue2012k} for lower bound analysis).\footnote{A finite-time guarantee requires more a refined analysis of $C$, and is an interesting direction for future work.}
The proof technique follows two major steps:
(1) prove the convergence of \multisparring as shown in Theorem \ref{thm:conv}; and
(2) bound the expected total regret for sufficiently large $T$.

 \begin{algorithm}[tb]
     \caption{\kersparring}
     \label{alg:ks}
 \begin{algorithmic}[1]
 	 \INPUT Input space $S$, GP prior $(\mu_0, \sigma_0)$, $m$ the number of arms drawn at each iteration   
     \FOR{$t=1,2,\ldots $}
     	\FOR{$j = 1, \ldots, m$}
          \STATE Sample $f_j$ from $(\mu_{t-1}, \sigma_{t-1})$
          \STATE Select $i_j(t) := \argmax_x{f_j(x)}$
    	\ENDFOR
        \STATE Play $m$ arms $\{i_j(t)\}_j$, observe pairwise feedback matrix $R = \{r_{jk} \in \{0,1,\emptyset\}\}_{m \times m}$
        \FOR{$j, k = 1, \ldots, m$}
          \IF{$r_{jk}\neq\emptyset$}
          \STATE apply Bayesian update using $(i_j(t), r_{jk})$ to obtain $(\mu_t, \sigma_t)$
          \ENDIF
        \ENDFOR
     \ENDFOR
 \end{algorithmic}
 \end{algorithm}

\subsection{Dependent Arms Case}
We use Gaussian processes (see Section \ref{sec:gp}) to model dependencies among arms. 
Applying Gaussian processes is not straightforward, since the underlying utility function is not directly observable or does not exist. 
We instead use Gaussian processes to model a specific the preference function. In Gaussian process notation, the preference function $f(b)$ represents the preference of choosing $b$ over the perfect ``environment'' of competing arms. Like in the independent arms case (Section \ref{sec:multisparring}), the perfect environment corresponds to having all the remaining arms be deterministically selected as the best arm $b_1$, yielding $f(b) = P(b \succ b_1)$. We model $f(b)$ as a sample from a Gaussian process $GP(\mu(b),k(b, b'))$. Note that this setup is analogous to the independent arms case, which uses a Beta prior to estimate the probability of each arm defeating the environment (and converges to competing against the best environment). 


Algorithm \ref{alg:ks} describes \kersparring, which instantiates \selfsparring using a Gaussian process Thompson sampling algorithm.
The input space $S$ can be continuous. 
At each iteration $t$, $m$ arms are sampled using the Gaussian process prior $D_{t-1}$. The posterior $D_t$ is then updated by the responses $R$ and the prior. 

Figure~\ref{fig:ks} illustrates the optimization process in a one-dimensional example. The underlying preference function against the best environment is shown in blue. Dashed lines are the mean function of GP. Shaded areas are $\pm 2$ standard deviations regions (high confidence regions). Figures \ref{fig:ks}(a)(b)(c) represent running \kersparring algorithm at 5, 20, and 100 iterations. The GP model can be observed to be converging to the preference function against the best environment.

We conjecture that it is possible to prove no-regret guarantees that scale w.r.t. the dimensionality of the kernel. However, there does not yet exist suitable regret analyses for Gaussian Process Thompson Sampling in the kernelized MAB setting to leverage.

\begin{table*}[t]
\centering
\begin{small}
\begin{tabular}{|l|l|}
\hline
Name  & Distribution of Utilities of arms \\ \hline
\hline
1good & 1 arm with utility 0.8, 15 arms with utility 0.2  \\ \hline
arith & 1 arm with utility 0.8, 15 arms forming an arithmetic sequence between 0.7 and 0.2 \\ \hline
\end{tabular}
\end{small}
\caption{16-arm synthetic datasets used for experiments.}
\label{tab:arms}
\vspace{-0.1in}
\end{table*}

\section{Experiments}\label{sec:experiments}

\subsection{Simulation Settings \& Datasets}
\label{sec:data}

\textbf{Synthetic Functions.} We evaluated on a range of 16-arm synthetic settings derived from  the utility-based dueling bandits setting of \citet{ailon2014reducing}. For the multi-dueling setting, we used the following preference functions:
\vspace{-0.1in}
\begin{table}[H]
\centering
\begin{tabular}{lc}
linear:  & $\phi(x, y) - 1/2 = (1+x-y)/2$          \\
logit:   & $\phi(x, y) - 1/2 = (1+\exp{(y-x)})^{-1}$
\end{tabular}
\end{table}
\vspace{-0.2in}
and the utility functions shown in Table \ref{tab:arms} (generalized from those in \citet{ailon2014reducing}). 
Note that although these preference functions do not satisfy approximate linearity over their entire domains, they do for the utility samples (over the a finite subset of arms).

\textbf{MSLR Dataset.}
Following the evaluation setup of \citet{brost2016multi}, we also  used the Microsoft Learning to Rank (MSLR) WEB30k dataset, which consists of over 3 million query-document pairs labeled with relevance scores \citep{liu2007letor}. Each pair is  scored along 136 features, which can be treated as rankers (arms). For any subset of arms, we can estimate a preference matrix using the expected probability over the entire dataset of one arm beating another using top-10 interleaving and a perfect-click model. 
We simulate user feedback by using team-draft multileaving \citep{schuth2014multileaved}.

\subsection{Vanilla Dueling Bandits Experiments}
We first compare against the vanilla dueling bandits setting of dueling a single pair of arms at a time.  These experiments are included as a sanity check to confirm that \selfsparring (with $m=2$) is a competitive algorithm in the original dueling bandits setting, and are not the main focus of our empirical analysis.

We empirically evaluate against a range of conventional dueling bandit algorithms, including:
\begin{itemize}
\vspace{-0.1in}
\item \textbf{Interleaved Filter (IF)} \citep{yue2012k}
\vspace{-.1in}
\item \textbf{Beat the Mean (BTM)} \citep{yue2011beat}
\vspace{-.1in}
\item \textbf{RUCB} \citep{zoghi2014relative}
\vspace{-.1in}
\item \textbf{MergeRUCB} \citep{zoghi2015mergerucb}
\vspace{-.1in}
\item \textbf{Sparring + UCB1} \citep{ailon2014reducing}
\vspace{-.1in}
\item \textbf{Sparring + EXP3} \citep{dudik2015contextual}
\vspace{-.1in}
\item \textbf{RMED1} \citep{komiyama2015regret}
\vspace{-.1in}
\item \textbf{Double Thompson Sampling} \citep{wu2016doublets}
\end{itemize}
For Double Thompson Sampling and \multisparring, we set the learning rates to be 2.5 and 3.5 as optimized over a separate dataset of uniformly sampled utility functions. We use $\alpha=0.51$ for RUCB/MergeRUCB, $\gamma=1$ for BTM, and $f(K) = 0.3K^{1.01}$ for RMED1.

\textbf{Results.}
For each scenario, we run each algorithm 100 times for 20000 iterations. For brevity, we show in Figure \ref{fig:best9} the average regret of one synthetic simulation along with shaded one standard-deviation areas.  We observe that \selfsparring is competitive with the best performing methods in the original dueling bandits setting.  More complete experiments that replicate \citet{ailon2014reducing} are provided in the supplementary material, and demonstrate the consistency of this result.

Double Thompson Sampling (DTS) is the best performing approach in Figure \ref{fig:best9}, which is a fairly consistent result in the extended results in the supplementary material.
However, given their high variances they are essentially comparable w.r.t. all other algorithms. 
Furthermore, \multisparring has the advantage of being easily extensible to the more realistic multi-dueling and kernelized settings, which is not true of DTS.



\begin{figure}[t]
\centering
\includegraphics[width=0.5\textwidth]{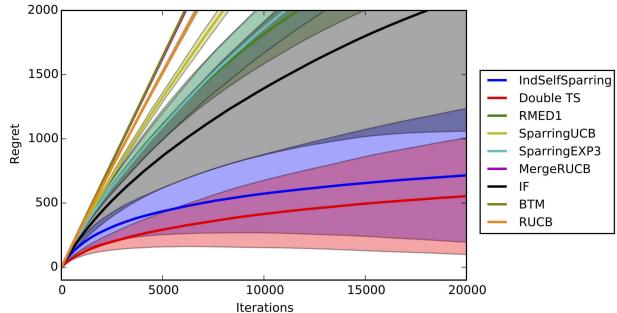}
\vspace{-0.1in}
\caption{Vanilla dueling bandits setting.  Average regret for top nine algorithms on logit/arith. Shaded regions correspond to one standard deviation.}
\label{fig:best9}
\end{figure}

\subsection{Multi-Dueling Bandits Experiments}

We next evaluate the multi-dueling setting with independent arms.  We compare against the main existing approaches that are applicable to the multi-dueling setting, including the MDB algorithm \citep{brost2016multi}, and the multi-dueling extension of Sparring, which we refer to as MultiSparring \citep{ailon2014reducing}.  
Following \cite{brost2016multi}, we use $\alpha=0.5$ and $\beta=1.5$ for the MDB algorithm. For \multisparring, we set learning rate to be the default 1.  Note that the vast majority dueling bandits algorithms are not easily applicable to the multi-dueling setting.  For instance, RUCB-style algorithms treat the two arms asymmetrically, which is not easily generalized to multi-dueling.

\textbf{Results on Synthetic Experiments.}
We test $m=4$ on the linear 1good and arith datasets in Figure \ref{fig:multi1good} and Figure \ref{fig:multiarith}, respectively. We observe that \multisparring significantly outperforms competing approaches.

\textbf{Results on MSLR Dataset.} Following the simulation setting of \cite{brost2016multi} on the MSLR dataset (see Section \ref{sec:data}), we compared against the MDB algorithm over the same collection of 50 randomly sampled 16-arm subsets. We ensured that each 16-arm subset had a Condorcet winner; in general it is likely for any random subset of arms in the MSLR dataset to have a Condorcet winner \citep{zoghi2015copeland}.  
Figure \ref{fig:multimslr} shows the results, where we again see that \multisparring enjoys significantly better performance.  

\begin{figure}[t]
\centering
\includegraphics[width=0.4\textwidth]{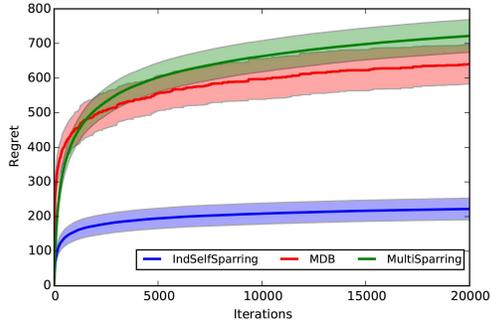}
\vspace{-0.1in}
\caption{Multi-dueling regret for linear/1good setting}
\label{fig:multi1good}
\vspace{-0.1in}
\end{figure}

\begin{figure}[t]
\centering
\includegraphics[width=0.4\textwidth]{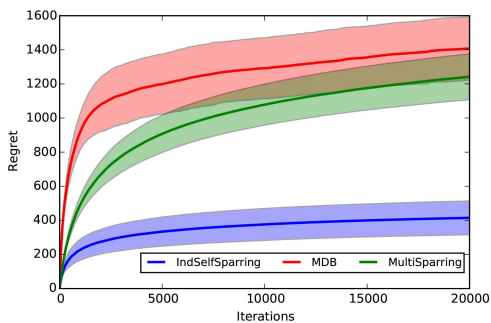}
\vspace{-0.1in}
\caption{Multi-dueling regret for linear/arith setting}
\label{fig:multiarith}
\vspace{-0.05in}
\end{figure}


\begin{figure}[t]
\centering
\includegraphics[width=0.42\textwidth]{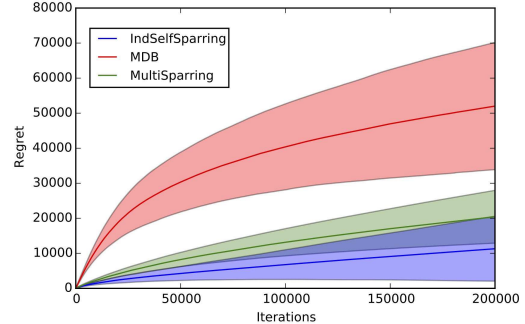}
\vspace{-0.1in}
\caption{Multi-dueling regret for MSLR-30K experiments}
\vspace{-0.14in}
\label{fig:multimslr}
\end{figure}

\begin{figure}[b]
\centering
\includegraphics[width=0.38\textwidth]{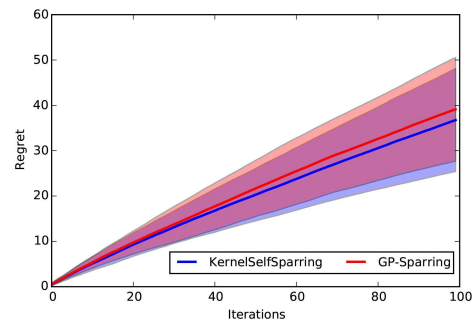}
\vspace{-0.15in}
\caption{2-dueling regret for kernelized setting with synthetic preferences}
\vspace{-0.05in}
\label{fig:gpsynthetic}
\end{figure}

\begin{figure}[t]
\centering
\includegraphics[width=0.38\textwidth]{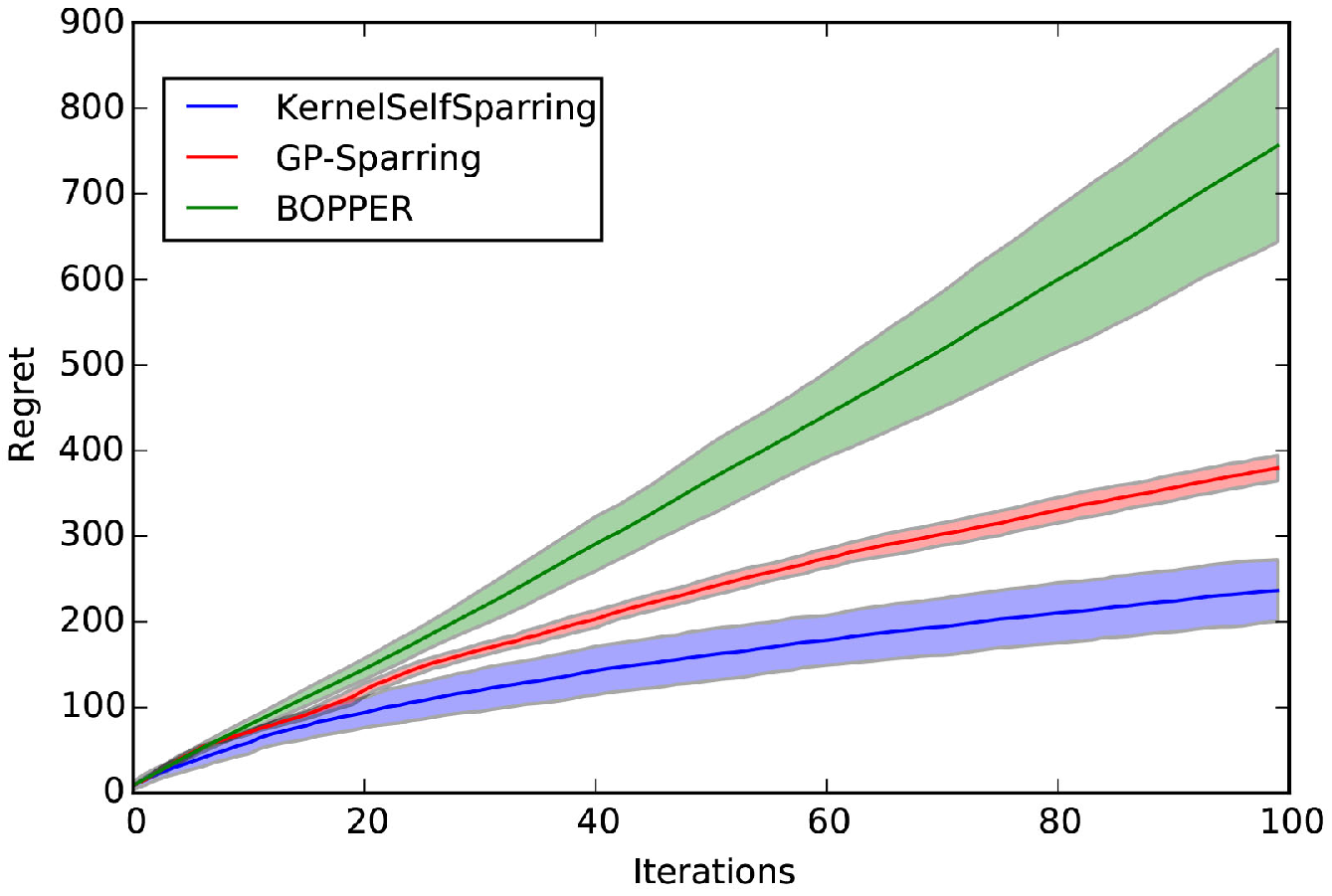}
\vspace{-0.1in}
\caption{2-dueling regret for kernelized setting with Forrester objective function}
\vspace{-0.05in}
\label{fig:gpforrester}
\end{figure}

\begin{figure}[t]
\centering
\includegraphics[width=0.38\textwidth]{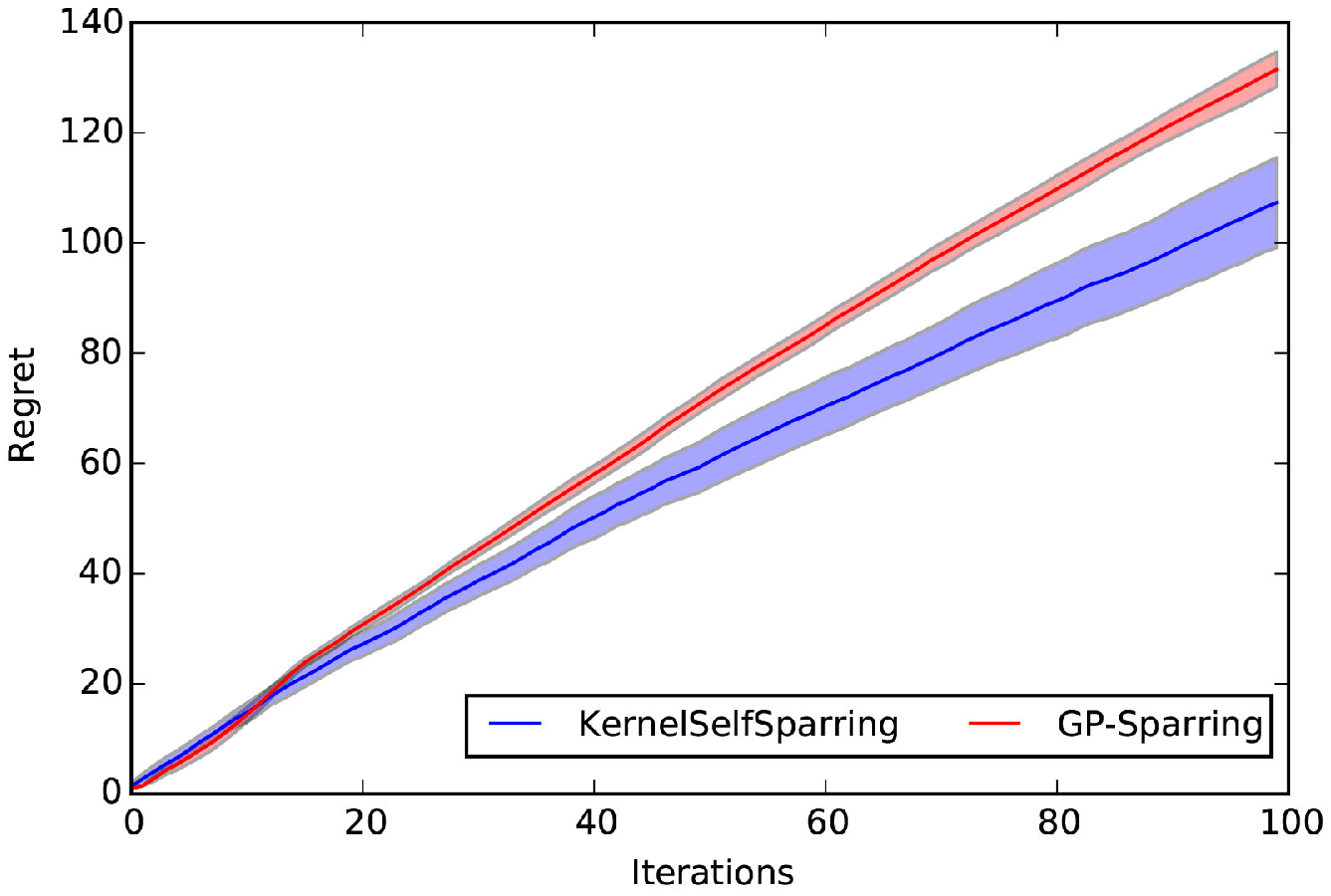}
\vspace{-0.1in}
\caption{2-dueling regret for kernelized setting with Six-Hump Camel objective function}
\vspace{-0.1in}
\label{fig:gpsixhump}
\end{figure}

\begin{figure}[t]
\centering
\includegraphics[width=0.39\textwidth]{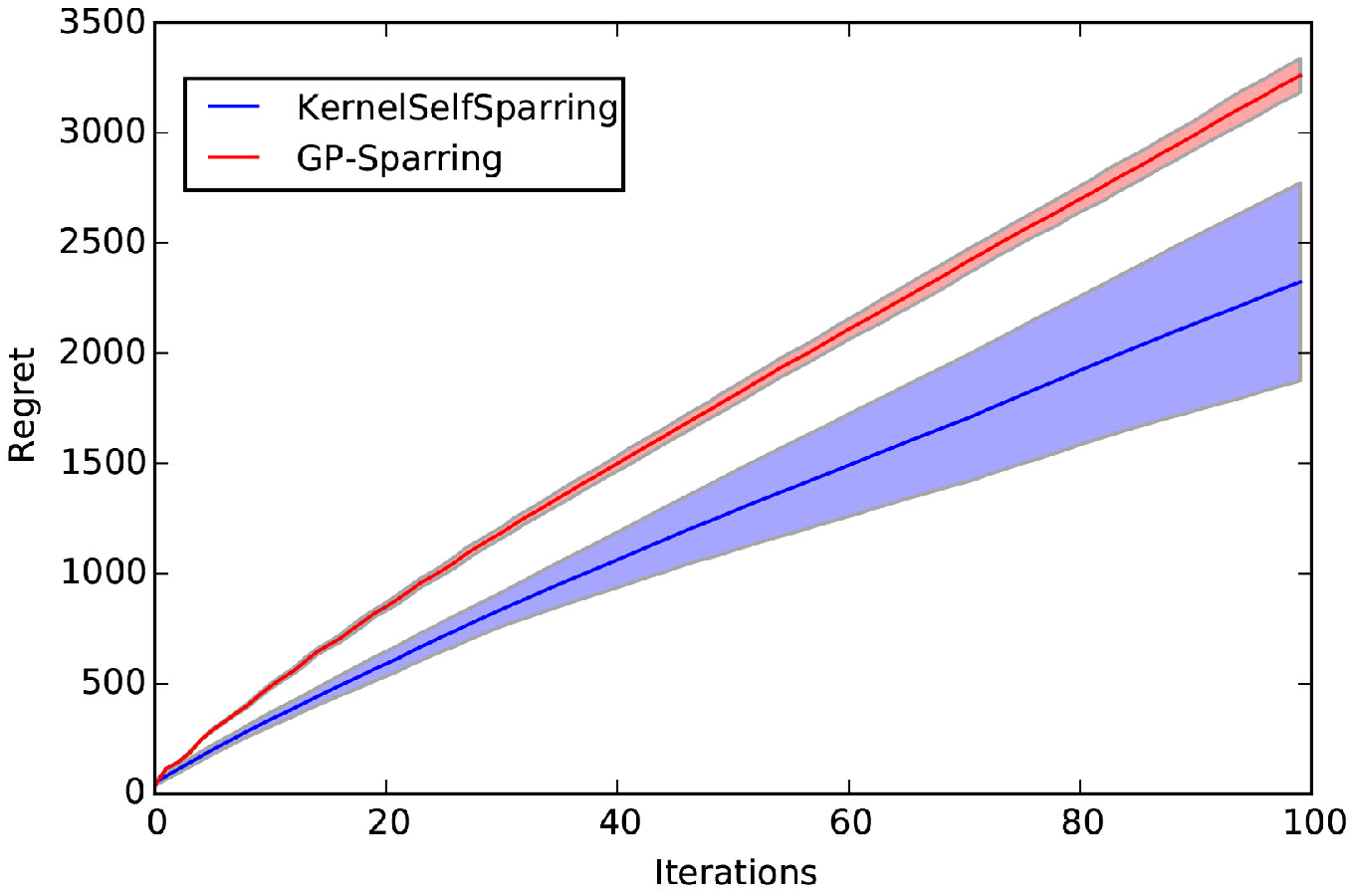}
\vspace{-0.1in}
\caption{Multi-dueling regret for kernelized setting with Forrester objective function}
\label{fig:multiforrester}
\vspace{-0.05in}
\end{figure}

\begin{figure}[t]
\centering
\includegraphics[width=0.38\textwidth]{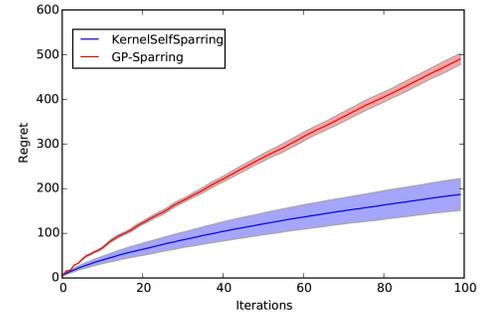}
\vspace{-0.1in}
\caption{Multi-dueling regret for kernelized setting with Six-Hump Camel objective function}
\label{fig:multisixhump}
\vspace{-0.05in}
\end{figure}




\subsection{Kernelized (Multi-)Dueling Experiments}

We finally evaluate the kernelized setting for both the 2-dueling and the multi-dueling case.
We evaluate \kersparring against BOPPER \citep{gonzalez2016bayesian} and Sparring \citep{ailon2014reducing} with GP-UCB \citep{srinivas10}. BOPPER is a Bayesian optimization method can be applied to kernelized 2-dueling setting (but not multi-dueling).  Sparring with GP-UCB, which refer to as GP-Sparring, is essentially a variant of our \kersparring approach but maintains a $m$ GP-UCB bandit algorithms (one controlling each choice of arm to be dueled), rather than just a single one.

\kersparring and GP-Sparring use GPs that model the preference function, i.e. are one-sided, whereas BOPPER uses a GP to model the entire preference matrix. Following \citet{srinivas10}, we use a squared exponential kernel with lengthscale parameter 0.2 for both GP-Sparring and \kersparring, and use a squared exponential kernel with parameter 1 for BOPPER. We initialize all GPs with a zero-mean prior, and use sampling noise variance $\sigma^2=0.025$. For GP-Sparring, we use the scaled-down version of $\beta_t$ as suggested by \citet{srinivas10}.
 
We use the Forrester and Six-Hump Camel functions as utility functions on $[0,1]$ and $[0,1]^2$, respectively, as in \citet{gonzalez2016bayesian}. Similarly, we use the same uniform discretizations of 30 and 64 points for the Forrester and Six-Hump Camel settings respectively, and use the logit link function to generate preferences.

Since the BOPPER algorithm is computationally expensive, we only include it in the Forrester setting, and run each algorithm 20 times for 100 iterations. In the Six-Hump Camel setting, we run \kersparring and GP-Sparring for 500 iterations 100 times each. Results are presented in Figures \ref{fig:gpforrester} and \ref{fig:gpsixhump}, where we observe much better performance from \kersparring against both BOPPER and GP-Sparring.


In the kernelized multi-dueling setting, we compare against GP-Sparring.  We run each algorithm for 100 iterations 50 times on the Forrester and Six-Hump Camel functions, and plot their regrets in Figures \ref{fig:multiforrester} and \ref{fig:multisixhump} respectively. We use $m=4$ for both algorithms, and the same discretization as in the standard dueling case.  We again observe significant performance gains of our \kersparring approach.



\section{Conclusions}



We studied multi-dueling bandits with dependent arms.  This setting extends the original dueling bandits setting by dueling multiple arms per iteration rather than just two, and modeling low-dimensional dependencies between arms rather than treat each arm independently.  Both extensions are motivated by practical real-world considerations such as in personalized clinical treatment \citep{sui2014clinical}.
We proposed \selfsparring, which is simple and easy to extend, e.g., by integrating with kernels to model dependencies across arms.  
Our experimental results demonstrated significant reduction in regret compared to state-of-the-art dueling bandit algorithms. Generally, relative benefits compared to dueling bandits increased with the number of arms being compared. For \selfsparring, the incurred regret did not increase substantially as the number of arms increased.

Our approach can be extended in several important directions.  Most notably, the theoretical analysis could be improved.  For instance, it would be more desirable to provide explicit finite-time regret guarantees rather than asymptotic ones.  Furthermore, an analysis of the kernelized multi-dueling setting is also lacking.  From a more practical perspective, we assumed that the choice of arms does not impact the feedback mechanism (e.g., all pairs), which is not true in practice (e.g., humans can have a hard time distinguishing very different arms).



\newpage
\begin{small}
\bibliography{selfspar}
\bibliographystyle{uai}
\end{small}

\newpage
\appendix

\section{Proofs}
\label{app:pf}
This section provides the proof sketch of Lemmas and Theorems mentioned in the main paper.

\textbf{Lemma~\ref{lem:ts}.}
For the K-armed stochastic MAB problem, Thompson Sampling has expected regret:
$\mathbb{E}[R_T^{\text{MAB}}] = \mathcal{O}\left(\frac{K}{\Delta}\ln T \right)$, where $\Delta$ is the difference between expected rewards of the best two arms.
\begin{proof}
This lemma is a direct result from Theorem 2 of \citet{agrawal2012} and Theorem 1 of \citet{kaufmann2012}.
\end{proof}

\textbf{Lemma~\ref{lem:io}.}
Running \multisparring with infinite time horizon will sample each arm infinitely often.
\begin{proof}
Proof by contradiction. \\
Let $B(x;\alpha,\beta) = \int_0^x t^{\alpha-1}(1-t)^{\beta-1} dt$. 
Then the CDF of Beta distribution with parameters $(\alpha,\beta)$ is $$F(x;\alpha,\beta) = \frac{B(x;\alpha,\beta)}{B(1;\alpha,\beta)}.$$ 
Suppose arm $b$ can only be sampled in finite number of iterations. Then there exists finite upper bound $T_b$ for $\alpha_b + \beta_b$. For any given $x \in (0,1)$, the probability of sampling values of arm $b$ $\theta_b$ greater than $x$ is $$P(\theta_b > x) = 1-F(x;\alpha_b,\beta_b) $$ 
$$\geq 1-F(x;1,T_b-1) = (1 - x)^{T_b-1} >0$$
Then by running \multisparring, the probability of choosing arm $b$ after it has been chosen $T_b$ times: 

$$P(\theta_b \geq max_i\{\theta_{b_i}\}) \geq \prod_i P(\theta_b \geq \theta_{b_i})$$
is strictly non-zero. That violates any fixed upper bound $T_b$.
\end{proof}

\textbf{Theorem~\ref{thm:conv}.}
Under Approximate Linearity, \multisparring converges to the optimal arm $b_1$ as running time $t\rightarrow \infty$: $\lim_{t\rightarrow \infty} \mathbb{P}(b_t = b_1) = 1$.
\begin{proof}
\multisparring keeps one Beta distribution $ Beta(\alpha_i(t),\beta_i(t))$ for each arm $b_i$ at time step $t$. Let $\hat{\mu}_i(t)= \frac{\alpha_i(t)}{\alpha_i(t) + \beta_i(t)}$, $\hat{\sigma}^2_i(t) = \frac{\alpha_i(t)\beta_i(t)}{(\alpha_i(t) + \beta_i(t))^2(\alpha_i(t) + \beta_i(t) + 1)}$ be the empirical mean and variance for arm $b_i$. \\
Obviously,  $\hat{\sigma}^2_i(t) \rightarrow 0$ as $(\alpha_i(t) + \beta_i(t)) = (S_i(t)+F_i(t)) \rightarrow \infty$. By Lemma~\ref{lem:io} we have $(S_i(t)+F_i(t)) \rightarrow \infty$ as $t\rightarrow \infty$. That shows every Beta distribution is concentrating to a Dirac function at $\hat{\mu}_i(t)$ when $t\rightarrow \infty$. Define $\hat{\mu}(t) = [\hat{\mu}_1(t), \cdots, \hat{\mu}_K(t)]^T \in [0,1]^K$to be the vector of means of all arms. Then $\mu = \{\mu_i = P(b_i \succ b_1)\}_{i = 1, \cdots, K}$ is a stable point for \multisparring in the $K$ dimensional mean space.

Suppose there exists another stable point $\nu \in [0,1]^K$($\nu \neq \mu$) for \multisparring, consider the following two possibilities: (1) $\nu_1 = max_i\{{\nu_i}\}$ and (2) $\nu_1 < max_i\{{\nu_i}\} = \nu_j$.

Since the Beta distributions for each arm $b_i$ is concentrating to Dirac functions at $\nu_i$, $P(\theta_i > \theta_j) \in [\mathbb{I}(\nu_i > \nu_j) -\delta, \mathbb{I}(\nu_i > \nu_j) + \delta]$ for any fixed $\delta > 0$ with high probability.

If (1) holds, then $\nu_1$ will converge to  $\frac{1}{2} = \mu_1$ and $\nu_i$ will converge to $P(b_i \succ b_1) = \mu_i$. Thus $\nu = \mu$. Contradict to $\nu \neq \mu$.

If (2) holds, then $\nu_j$ will converge to  $\frac{1}{2} = \mu_1$ and $\nu_1 \in [P(b_1 \succ b_j) - \delta, P(b_1 \succ b_j) + \delta]$ for any fixed $\delta>0$ with high probability. Since $P(b_1 \succ b_j) \geq \frac{1}{2}+ \Delta$, $\nu_1 \in [P(b_1 \succ b_j) - \delta, P(b_1 \succ b_j) + \delta] \geq \frac{1}{2} + \Delta - \delta$ . Since $\delta$ can be arbitrarily small, we have $\nu_1 \geq \frac{1}{2} + \Delta-\delta > \frac{1}{2}+\delta > \nu_j$. That contradict to $\nu_1 < \nu_j$.

In summary, $\mu = \{ \mu_i =P(b_i \succ b_1)\}_{i = 1, \cdots, K}$ is the only stable point in the mean space. As $\hat{\mu}(t) \rightarrow \mu$, $\mathbb{P}(b_t = b_1) \rightarrow 1$.

Define $\mathbb{P}_t = [P_1(t), P_2(t), ..., P_K(t)]$ as the probabilities of picking each arm at time $t$. Let $\mathbb{P} = \{\mathbb{P}_t\}_{t = 1, 2, ...}$ be the sequence of probabilities w.r.t. time. 
Assume \multisparring is non-convergent. It is equivalent to say that $\mathbb{P}$ is not converging to a fixed distribution. Then $\exists \delta > 0$ and arm $i$ s.t. the sequence of probabilities $\{P_i(t)\}_t$ satisfies:
$$\limsup_{t \rightarrow \infty} P_i(t) - \liminf_{t \rightarrow \infty} P_i(t) > \delta$$
w.h.p. which is equivalent of having:
$$\limsup_{t \rightarrow \infty} \hat{\mu}_i(t) - \liminf_{t \rightarrow \infty} \hat{\mu}_i(t) > \epsilon$$
w.h.p. for some fixed $\epsilon > 0$. This violates the stability of \multisparring in the $K$ dimensional mean space as shown above. So as $t\rightarrow \infty$,   $\hat{\mu}(t) \rightarrow \mu$, $\mathbb{P}(b_t = b_1) \rightarrow 1$.

\end{proof}

\textbf{Lemma~\ref{lem:fixed}.}
Under Approximate Linearity, selecting only one arm  via Thompson sampling against a fixed distribution over the remaining arms leads to optimal regret w.r.t. choosing that arm.
\begin{proof}
We first prove the results for $m = 2$. Results for any $m > 2$ can be proved in a similar way.

Consider Player 1 drawing arms from a fixed distribution $L$. Player 2's drawing strategy is an MAB algorithm $\mathcal{A}$.

Let $R_A(T)$ be the regret of algorithm $\mathcal{A}$ within horizon $T$. $B(T) = \sup \mathbb{E}[R_A(T)]$ is the supremum of the expected regret of $\mathcal{A}$. 

The reward of Player 2 at iteration $t$ is $\phi (b_{2t}, b_{1t})$. Reward of keep playing the optimal arm is $\phi (b_1, b_{1t})$. So the total regret after $T$ rounds is 
$$R_A(T) = \sum_{t=1}^{T} [\phi (b_1, b_{1t}) - \phi (b_{2t}, b_{1t})]$$
Since Approximate Linearity yields
$$\phi (b_1, b_{1t}) - \phi (b_{2t}, b_{1t}) \geq \gamma \cdot \phi(b_1, b_{2t})$$
We have
$$\mathbb{E}[R_A(T)] = \mathbb{E}\mathbb{E}_{b_{1t}\sim L}\left[\sum_{t=1}^{T} [\phi (b_1, b_{1t}) - \phi (b_{2t}, b_{1t})\right]$$
$$\geq \mathbb{E}\mathbb{E}_{b_{1t}\sim L} \left[\sum_{t=1}^{T} \gamma \cdot \phi(b_1, b_{2t})\right] $$ 
$$= \gamma \cdot \mathbb{E} \left[ \sum_{t=1}^{T}\phi(b_1, b_{2t})\right]
= \gamma \cdot \mathbb{E}[R(T)]$$
So the total regret of Player 2 is bounded by
$$\mathbb{E}[R(T)] \leq \frac{1}{\gamma} \mathbb{E}[R_A(T)] \leq \frac{1}{\gamma} \sup \mathbb{E}[R_A(T)] = \frac{1}{\gamma} B(T)$$
\end{proof}

\begin{cor}
\label{cor:conv}
If approximate linearity holds, competing with a drifting but converging distribution of arms guarantees the one-side convergence for Thompson Sampling.
\end{cor}

\begin{proof}
Let $D_t$ be the drifting but converging distribution and $D_t \rightarrow D$ as $t \rightarrow \infty$. Let $b_T$ be the drifting mean bandit of $D_T$ after $T$ iterations. Since $D_t$ is convergent, $\exists T > K$ such that 
$$\phi(\sup_{t > T} b_T,  \inf_{t > T} b_T) < \phi(b_1, b_2)$$
where $\phi(b_1, b_2)$ is the preference between the best two arms. 
The mean value of feedback by playing arm $i$ is $\phi(b_i, b_T)$. If $b_T$ is fixed, by Lemma\ref{lem:fixed}, Thompson sampling converges to the arm:
$i^* = \argmax_i \phi(b_i, b_T)$. For drifting $b_T$, define $b^+ = \sup_{t > T} b_T$ and $b^- = \inf_{t > T} b_T$.

Thompson sampling convergence to the optimal arm implies that:
$$\phi(b_1, b^+) > \phi(b_i, b^-)$$
for all $i \neq 1$. Consider:
$$\phi(b_1, b^+) - \phi(b_2, b^-)$$
$$ = \phi(b_1, b^+) - \phi(b_2, b^-) + \phi(b_1, b^-) - \phi(b_1, b^-)$$
$$ = \phi(b_1, b^-) - \phi(b_2, b^-) + \phi(b_2, b^+) - \phi(b_1, b^-)$$
$$ \geq \gamma\cdot[\phi(b_1, b_2) - \phi(b^+, b^-)] > 0$$
by approximate linearity.

So we have $\phi(b_1, b^+) > \phi(b_2, b^-)$. Since $\phi(b_2, b^-) > \phi(b_i, b^-)$ for $i > 2$. Then we have $$\phi(b_1, b^+) > \phi(b_i, b^-)$$
holds for all $i \neq 1$. So Thompson sampling converge to the optimal arm.
\end{proof}

\textbf{Theorem~\ref{thm:ms}.}
Under Approximate Linearity, \multisparring converges to the optimal arm with asymptotically optimal no-regret rate of $\mathcal{O}(K\ln(T)/\Delta)$.
Where $\Delta$ is the difference between the rewards of the best two arms.

\begin{proof}
Theorem~\ref{thm:conv} provides the convergence guarantee of \multisparring. Corollary~\ref{cor:conv} shows one-side convergence for playing against a converging distribution.

Since \multisparring converges to the optimal arm $b_1$ as running time $t\rightarrow \infty$: $\lim_{t\rightarrow \infty} \mathbb{P}(b_t = b_1) = 1$. For $\forall \delta>0$, there exists $C(\delta) > 0$ such that for any $t > C(\delta)$, the following condition holds w.h.p.:
$P(b_t = b_1) \ge 1-\delta$.

For the triple of bandits $b_1 \succ b_i \succ b_K$, Approximate Linearity guarantees:
$$\phi(b_i,b_K) < \phi(b_1,b_K) \leq \omega$$
holds for some fixed $\omega > 0$ and $\forall i \in \{2, \cdots, K-1\}$. With small $\delta$, the competing environment of any Player $p$ is bounded. If $\delta < \frac{\Delta}{\Delta + \omega}$,  $(1-\delta)\cdot(-\Delta) + \delta \cdot  \phi(b_2,b_K) < 0 = 1\cdot \phi(b_1,b_1)$. The competing environment can be considered as unbiased and the theoretical guarantees for Thompson sampling for stochastic multi-armed bandit is valid (up to a constant factor).

Then \multisparring has an no-regret guarantee that asymptotically matches the optimal rate of $\mathcal{O}(K\ln(T)/\Delta)$ up to constant factors, which proves Theorem~\ref{thm:ms}.

\end{proof}

\section{Further Experiments}
\label{app:exp}

\begin{figure*}[t]
\centering
\includegraphics[]{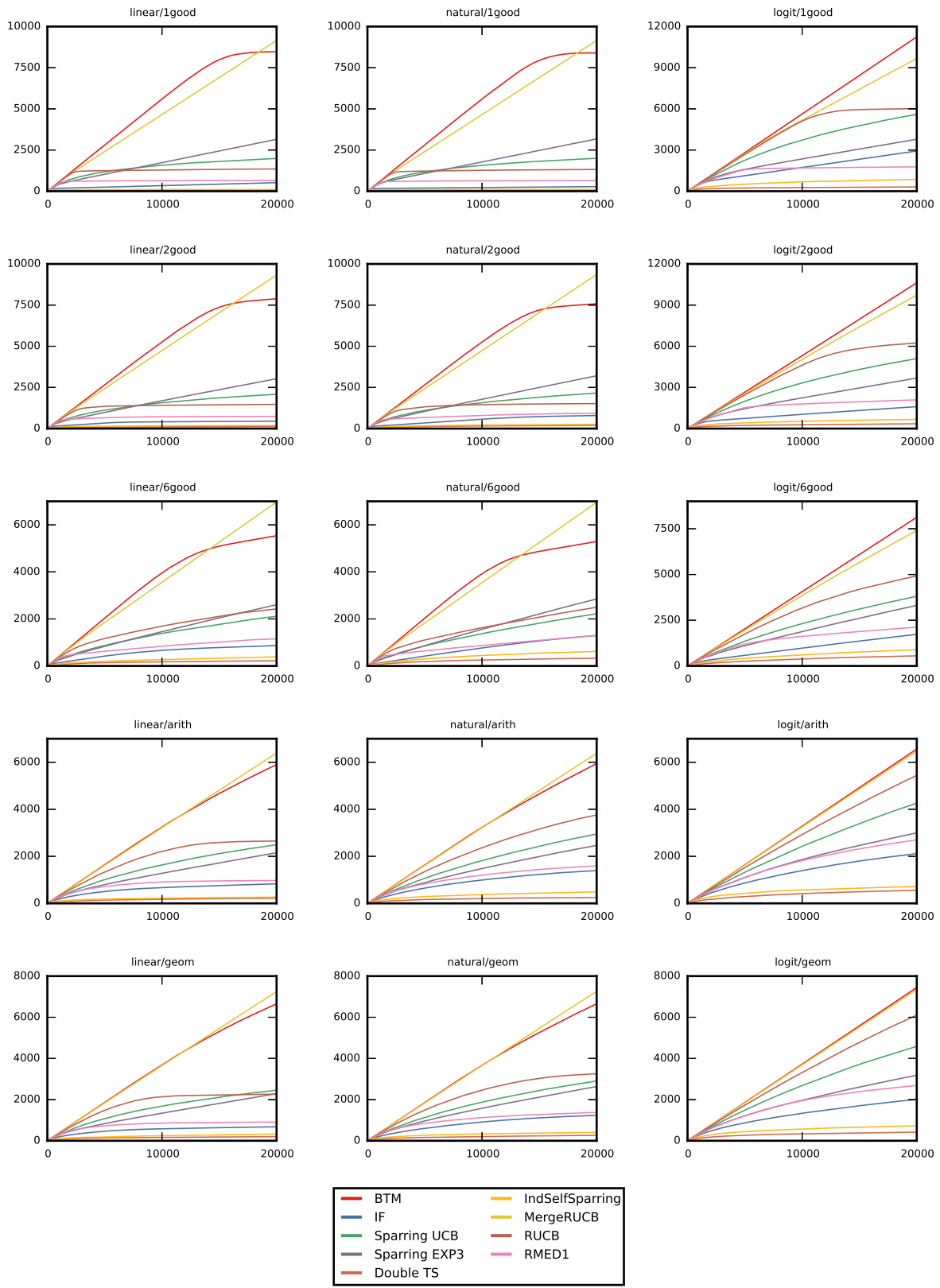}
\vspace{-0.1in}
\caption{Average regret vs iterations for each of 8 algorithms and 15 scenarios.}
\label{fig:grid}
\end{figure*}

\newpage
\begin{table*}[t]
\centering
\begin{tabular}{|l|l|}
\hline
Name  & Distribution of Utilities of arms \\ \hline
\hline
1good & 1 arm with utility 0.8, 15 arms with utility 0.2  \\ \hline
2good & 1 arm with utility 0.8, 1 arms with utility 0.7, 14 arms with utility 0.2  \\ \hline
6good & 1 arm with utility 0.8, 5 arms with utility 0.7, 10 arms with utility 0.2   \\ \hline
arith & 1 arm with utility 0.8, 15 arms forming an arithmetic sequence between 0.7 and 0.2 \\ \hline
geom  & 1 arm with utility 0.8, 15 arms forming a geometric sequence between 0.7 and 0.2 \\ \hline
\end{tabular}
\caption{16-arm synthetic datasets used for experiments.}
\label{tab:arms2}
\vspace{-0.1in}
\end{table*}

\end{document}